 \documentclass[smallabstract,smallcaptions]{dccpaper}

\usepackage{epsfig}
\usepackage{amsmath}
\usepackage{amssymb}
\usepackage{color}
\usepackage{url}
\usepackage{algorithm}
\usepackage{algorithmic}
\usepackage{booktabs}
\newlength{\figurewidth}
\newlength{\smallfigurewidth}

\newtheorem{definition}{Definition}

\newtheorem{theorem}{Theorem}
\newtheorem{proof}{Proof}

\begin{document}

\title
{\large
\textbf{On The Energy Statistics of Feature Maps in Pruning of Neural Networks with Skip-Connections\small\thanks{This work has been supported in part by the Army Research Office grant No. W911NF-15-1-0479.}}
}

\author{%
Mohammadreza Soltani$^{\ast}$, Suya Wu$^{\ast}$, Yuerong Li$^{\ast}$ \\ Jie Ding$^{\dag}$, and Vahid Tarokh$^{\ast}$\\[0.5em]
{\small\begin{minipage}{\linewidth}\begin{center}
\begin{tabular}{cc}
$^{\ast}$Department of Electrical and Computer Engineering, $^{\dag}$ School of Statistics &  \\
$^{\ast}$Duke University, $^{\dag}$University of Minnesota &  \\
\\
\end{tabular}
\end{center}\end{minipage}}
}

\maketitle
\thispagestyle{empty}

\begin{abstract}
We propose a new structured pruning framework for compressing Deep Neural Networks (DNNs) with skip-connections, based on measuring the statistical dependency of hidden layers and predicted outputs. The dependence measure defined by the energy statistics of hidden layers serves as a model-free measure of information between the feature maps and the output of the network. The estimated dependence measure is subsequently used to prune a collection of redundant and uninformative layers. Model-freeness of our measure guarantees that no parametric assumptions on the feature maps distribution are required, making it computationally appealing for very high dimensional feature space in DNNs. Extensive numerical experiments on various architectures show the efficacy of the proposed pruning approach with competitive performance to state-of-the-art methods.
\end{abstract}

\vspace{-.3cm}
\section{Introduction}
\label{intro}
\vspace{-.2cm}
Modern deep convolutional neural networks (CNNs) with competitive performance in image understanding competitions (e.g., \textit{ImageNet} \cite{deng2009imagenet} and \textit{COCO}~\cite{lin2014microsoft}) are over-parameterized with millions of parameters, resulting in a large memory utilization and computational complexity~\cite{krizhevsky2012imagenet,simonyan2014very}. 
As a result, deploying deep models with high accuracy in devices with limited hardware is a challenge, due to both memory requirements and a large number of floating point operations per second (FLOPs).
Many empirical observations have also demonstrated that the test accuracy of DNNs is not affected drastically by distorting their trained weights and even removing the hidden units~\cite{cheng2017survey, huang2018condensenet}. 
Balancing the trade-off between the size of a deep model and achieving high performance has recently fueled research activity in developing various model compression techniques, and quantifying the notion of redundancy in DNNs.

Among existing algorithms, 
a higher compression ratio. In general, pruning techniques are divided into two categories: Structured and Unstructured methods. Unstructured methods directly remove individual weights in a deep model, 
while structured methods apply pruning in the level of filters and layers by imposing some structure on the topology of weights. Unstructured methods provide a higher compression ratio and accuracy in terms of memory requirement, but less gain in the FLOPs number; however, structured methods are much easier to be deployed in parallel processors, and usually provide a significant reduction in the FLOPs count.
In this work, we study structured pruning techniques for deep networks with skip-connections e.g., ResNet~\cite{he2016deep} and DenseNet~\cite{huang2017densely} by finding a set of informative hidden units. We cast the pruning problem as a feature selection problem
These architectures are very robust against removing a set of layers referred to here as \emph{skip-units} (please see Section 3), making them ideal candidates for pruning. In order to select the most informative feature maps, we need to measure the dependency between these features and the output of the model. Measuring the dependency between random variables is an important problem in statistics, information theory, and machine learning~\cite{poczos2012copula}, which can be used to determine the information of one random variable about the other one. The most well-known dependence measure is the Shannon mutual information. Other measures include  Maximum Mean Discrepancy (MMD)~\cite{borgwardt2006integrating}, and Energy Dependence (Distance-based Correlation)~\cite{szekely2007measuring}, to name a few. 


In this paper, we use Energy Dependence as a model-free measure between a set of hidden layers and the output of a DNN with skip-connections as a feature selection criterion in order to assess the importance and relevance of the feature maps. Using this information, we propose a structured pruning algorithm to iteratively remove a set of less important redundant layers from the network. 

\vspace{-.2cm}
\section{Related Work}
\label{prior}
\vspace{-.2cm}
Here, we only review the  pruning techniques most relevant to our method. 
Among existing approaches in network compression, pruning techniques are quite popular~\cite{ xiao2019autoprune, zhuang2018discrimination, lin2018accelerating, malach2020proving, lin2020hrank,herrmann2020channel,ganesh2020mint}. These methods reduce the network complexity through weights/kernels removal. 

Pruning through the removal of weights is accomplished by determining and removing less important weights in the final prediction of the model. However, finding the less important weights is a prohibitively costly problem due to its combinatorial nature. To get around this issue, various criteria such as the minimum norm of weights, the activation of the feature maps, information gain, and etc have been proposed in literate~\cite{molchanov2016pruning}. However, these approaches often either ignore the dependency between layers or solve a more difficult optimization objective than the original loss. Evaluating the importance of a filter just from its weights is not a reliable measure as dependency between filters is not taken into account. In addition, higher weight values do not necessarily mean the true importance of a filter since the filter’s contribution can be compensated elsewhere in the network~\cite{ganesh2020mint}. 

The most relevant pruning technique with our approach is the structured method of dynamic path selection. In this method, a sub-graph or a group of weights with a specific structure is selected from the original network such that the chosen graph has a comparable inference (test) performance with the original model. The common practice is to train an additional module which selects where to forward between channels, blocks,
or other parts of the graph~\cite{liu2018dynamic, hu2018squeeze, wu2018blockdrop, wang2018skipnet, lin2020hrank, lee2020urnet}. In particular, approaches proposed by~\cite{wang2018skipnet} and \cite{wu2018blockdrop} have mostly focused on the residual networks, and their goal is to dynamically select a set of layers using reinforcement learning. however, they are less effective with deeper neural networks. In addition, the additional module may increase the complexity and running time of the whole model. 


\section{Background and Proposed Method}
\vspace{-.2cm}
\subsection{Deep Neural Models with Skip-Connection}
\label{skip_conn}
\vspace{-.2cm}
As mentioned, architectures with skip-connections are more robust to weight distortion than the feed-forward networks~\cite{veit2016residual}. Here, we take this phenomenon one step further and develop a structured pruning technique based on removing a collection of redundant layers instead of individual random layers. First, we define a \emph{skip-unit} as a set of layers where its output is a function of sequential application of operations in the unit including \textit{Conv}, \textit{Pooling}, \textit{ReLU}, \textit{BN}, \textit{Dropout}, etc and feature maps of previous unit(s). Specifically, consider a DNN with $L$ skip-units. We denote the input of $l^{th}$ unit as $U_{l-1}$. Let $T_l = f_l(U_{l-1})$ denote the output of sequential application of the aforementioned operations in a skip-unit, summarized by $f_l$. We call $T_l$ as the feature map of the unit $l$. Hence, each skip-unit is mathematically expressed by:
$
U_l = \Psi(T_l, U_{l:l-1}, \alpha_l), \ \  l=2,3,\ldots,L,
$
where $\alpha_l\in\{0,1\}$ is referred to as  \emph{policy} for unit $l$, and $U_{1:l-1}$ denotes all the skip-units from unit $1$ to the $(l-1)^{th}$ unit.  $\Psi$ also denotes an operation that combines $T_l$ and $U_{1:l-1}$. In ResNet and DenseNet models, $\Psi_{res}$ and $\Psi_{den}$ are respectively given by
\begin{align}
&U_l = \Psi_{res}(T_l, U_{1:l-1}, \alpha_l) = \alpha_l T_l + \mathcal{A}_{l-1}U_{l-1}, \\
&\hspace{0cm}U_l = \Psi_{dense}(T_l, U_{1:l-1}, \alpha_l) = Con(\alpha_l T_l, \ U_{1:l-1}), 
\end{align}
where $Con$ is a concatenation operation, and $\mathcal{A}_{l-1}$ is an identity, down-sampling, or some operation such as convolution operator(s). Furthermore, $U_0$, the input for the first unit is usually given by a convolution operator (e.g., in ResNet, and DenseNet). 
 \vspace{-.8cm}
\subsection{Model-Free Measurement of Dependency}
\vspace{-.2cm}
Our main idea for compression of deep models with skip-units is based on selecting the most informative hidden units by measuring their statistical dependency to the output of the model and then removing less important units. 
A popular information measure is Shannon's \emph{Mutual Information} (MI). Computing the mutual information between the output and the features of a machine learning model has been a research topic with a long history in unsupervised feature learning~\cite{linsker1988self}. In the context of deep learning, MI between the hidden layers and the output of a network has been investigated by the \emph{Information Bottleneck} theory~\cite{tishby2015deep}. While Shannon MI is an appealing mathematical measure, its estimation may be a difficult task. The estimation of MI for high-dimensional data needs a model assumption for the underlying probability distribution. The selected model may be statistically unrealistic. 
Here, we use a model-free measure with similar properties of MI that is not computationally intense~\cite{bottou2018geometrical}. We introduce a measure that quantifies the dependency between two random vectors $T$ and $Y$, a feature map and output labels, respectively. Our proposed measure is based on the Energy Distance between two random vectors.
\vspace{-.2cm}
\begin{definition}[Energy Distance (ED)~\cite{szekely2013energy}]:
    Suppose that the characteristic functions and distribution functions of two random vectors $T_i\in \mathbb{R}^d$ with finite first moment (i.e., $E||T_i||<\infty$) are respectively given by $\phi_i(\cdot)$ and $F_i(\cdot)$ ($i=1,2$). The ED between $F_1$ and $F_2$ is given by
    $
        \mathcal{E}(F_1,F_2) = \frac{1}{c_d}\int_{\mathbb{R}^d} \frac{|\phi_1(s)-\phi_2(s)|^2}{\|s\|^{d+1}}ds
        \label{eq11},
    $
    where  $c_d=\pi^{(1+d)/2}/\Gamma((1+d)/2)$ only depends on $d$.
    \vspace{-.2cm}
\end{definition}
\begin{theorem}\label{them1}
    The energy distance in~\eqref{eq11} can be written as~\cite{szekely2013energy}
    \small
    \begin{align}
        \hspace{-.19cm}\mathcal{E}(F_1,F_2)
        &=
        2E \|T_1-T_2\| - E \|T_1-T_1' \| - E \|T_2-T_2'\|, \label{eq10}
    \end{align}
    \normalsize
    where $T_1'$ and $T_2'$ are i.i.d copies of $T_1$ and $T_2$, respectively, and $\|\cdot\|$ denotes the Euclidean norm. Energy Distance is non-negative and $\mathcal{E}(F_1,F_2)=0$ iff $F_1=F_2$.
\end{theorem}
\vspace{-.2cm}
Now suppose that we have $n_i$ observations of $T_i$, $i=1,2$, where $t_{i,j}$ denotes the $j$-th observation of  $T_i$. Then an unbiased  estimator of ED in (\ref{eq10}) is given by:
\begin{align}
    \hat{\mathcal{E}}(F_1,F_2)
    &=
        \frac{2}{n_1n_2} \sum_{1\leq j_1\leq n_1,1\leq j_2\leq n_2}\|t_{1,j_1}-t_{2,j_2}\| -  \frac{1}{n_1^2}\sum_{1\leq j_1,j_1'\leq n_1} \|t_{1,j_1}-t_{1,j_1'}\|\nonumber\\
    &\hspace{4cm} -\frac{1}{n_2^2}\sum_{1\leq j_2,j_2'\leq n_2} \|t_{2,j_2}-t_{2,j_2'}\|.\label{eq14}
\end{align}
Applying the theory of V-statistics~\cite{lee2019u}, it can be proved under mild assumptions that the above $\hat{\mathcal{E}}(F_1,F_2)$ is a consistent estimator of $\mathcal{E}(F_1,F_2)$, that is 
$
\hat{\mathcal{E}}(F_1,F_2) \rightarrow \mathcal{E}(F_1,F_2)
$
in probability as $n=\min\{n_1,n_2\} \rightarrow \infty$. This implies that with sufficiently large data size, $\hat{\mathcal{E}}(F_1,F_2)$ vanishes for $T_1$ identically distributed as $T_2$, and is bounded away from zero otherwise, without the need of specifying any parametric model for the underlying probability distribution. We now define the dependency.
\vspace{-.3cm}

\begin{definition}[Energy Dependence]\label{definition1}
    Consider random vectors/variables $T \in \mathbb{R}^d$ and $Y \in \mathcal{Y}$, where 
    $\mathcal{Y}$ has $p$ elements, say $\{1,2,\ldots,p\}$ for the notational convenience. The Energy Dependence between $T$ and $Y$ is defined by $
       D(T,Y) = \max_{1\leq i,j \leq p} \mathcal{E}(F_i,F_j),$ where $F_i$ denotes the  distribution of $T$  conditional on $Y=i$.
       \vspace{-.3cm}
\end{definition}
\begin{theorem}
    For random vector/variable $T \in \mathbb{R}^d$ and $Y \in \mathcal{Y}$ with a finite alphabet  $\mathcal{Y}$, $D(T,Y)=0$ if $T$ and $Y$ are independent. 
    \vspace{-.3cm}
\end{theorem}
\begin{proof}
    The independence of $T$ and $Y$ is equivalent to $F_{T \mid Y=j} = F_{T \mid Y=j'}$ for any $1\leq j,j'\leq p$.
    This is further equivalent to $\mathcal{E}(F_i,F_j)=0$, which implies $D(T,Y)=0$, according to Definition~\ref{definition1} and Theorem~\ref{them1}. 
    \vspace{-.3cm}
\end{proof}
The larger the Energy Distance, the larger is the  dissimilarity. As a consequence, the Energy Dependence in Definition~\ref{definition1} may be interpreted as a quantification of the dependency between $T$ and $Y$. Based on (\ref{eq14}), we define a consistent estimator of $D(T,Y)$ by
$\hat{D}(T,Y) = \max_{1\leq i,j \leq p} \hat{\mathcal{E}}(F_i,F_j),$ where $\hat{\mathcal{E}}(F_i,F_j)$ is given by \eqref{eq14}. We use this to measure the statistical dependency between $T$ and $Y$. The larger its value, the more information $T$ reveals about $Y$. Similar to the consistency of (\ref{eq14}), $\hat{D}(T,Y)$ is a consistent estimator of $D(T,Y)$ according to the standard theory of V-statistics. 
\vspace{-.8cm}
\subsection{Proposed Algorithm}
\label{Alg}
\vspace{-.2cm}
The general idea is to measure the relevance and importance of different skip-units to the output of the network, and then select the less important ones. If we show the Energy Dependence between the unit $T_l$ and the output of the network, $Y$ by $D(T_l,Y), \ l=1,2,\dots,L$, ideally, we need to  to select a set $S^*\subseteq\{T_1, T_2,\dots,T_L\}$ of informative units, maximizing the average dependencies:
    $S^* = \underset{S}{\mathrm{argmax}} \frac{1}{|S|}\sum_{T_l\in S} D(T_l,Y),$
where $|S|$ demotes the size of the set $S$. However, the solution of this optimization problem is a set with highly redundant units which does not take into account the dependency between the units. In addition, due to the fact that the existence of the skip-units in ResNet and DenseNet architectures, the Markov Chain property among the layers is violated. This in turn implies that the mutual information between the skip-units and the output does not decrease monotonically according to the information processing inequality. As a result, selecting units with the largest energy values (information) may remove some important skip-units. We will show an experiment to verify this argument in below. Hence, we need to solve alternatively the following optimization problem to simultaneously maximize the relevance of units and minimizing the redundancy among them:
\begin{align*}
    S^* = \underset{S}{\mathrm{argmax}} \frac{1}{|S|}\sum_{T_l\in S} D(T_l,Y)-\frac{1}{|S|^2}\sum_{T_l,T_{l'}\in S} D(T_l,T_{l'}).
\end{align*}
However, solving the above problem is computationally expensive if it is not possible due to its combinatorial nature; as a result, we propose our algorithm which we call \emph{Pruning with Energy Dependence} (PED). The pseudocode of our algorithm is given in Algorithm~\ref{alg1}. The idea of PED can be summarized by 3 steps: (I) Clustering the Energy Dependence of units, (II) Selecting the most informative unit as a cluster head in each cluster with the largest Energy Dependence value, and (III) Removing other units from each cluster in order to prune the redundant units. In  Algorithm~\ref{alg1} $|A|$ is the size of a set $A$ and $S^t$ is the index set of active skip-units at stage $t$, and  $\mathbf{DNN}^t$ means a trained pruned model with weights initialized from stage $t-1$. 


In each stage $t$, PED determines the units whose feature maps $T_l$ are more informative about the output of the model for the given input by computing $\hat{D}(T_l,Y)$, $l=1,2,\ldots,L$, defined in Definition~\ref{definition1}. Next, PED uses $\hat{D}(T_l, Y)$ to select active skip-units for which $\alpha_l=1$ by clustering the $\hat{D}(T_l, Y)$ values. For clustering purpose, we have used an optimal $k$-means algorithm based on dynamic programming~\cite{wang2011ckmeans} (by invoking $\mathrm{Clust}(K^t, \mathbf{D}^t, S^t)$ sub-routine)\footnote{Other clustering algorithms can be used here; however, since the clustering is applied on 1-D data, using the approach by~\cite{wang2011ckmeans} guarantees optimality of clustering.}. 
\begin{algorithm}[t]
\begin{algorithmic}
\caption{Pruning with Energy Dependence (PED)}
\label{alg1}
\small{
\STATE \textbf{INPUT:}
\STATE \ \ \ \ \ $\mathbf{DNN}^0$:  Pre-trained Deep Neural Network
\STATE \ \ \ \ \ $S^0$: The index set  of skip-units in $\mathbf{DNN}^{0}$
\STATE \ \ \ \ \ $T_l^0$: Feature maps at $t=0$, \ $l=1,2,\dots,|S^0|$
\STATE \ \ \ \ \ $K^0$: Number of Clusters at $t=0$
\STATE \ \ \ \ \ $N$: Number of stages
\FOR {$t =0,1,\ldots,N-1$}
\STATE Compute $\hat{D}(T_l^t, Y), \ \ l=1,\dots,|S^t| \ $ using $\mathbf{DNN}^{t}$   
\STATE Construct $\mathbf{D}^t = [\hat{D}(T_1^t, Y), ,\ldots,\hat{D}(T_{|S^t|}^t, Y)]$ 
\STATE 
\vspace{-.6cm}
\begin{align*}
&\hspace{-.40cm}\{Cluster_1,\ldots,Cluster_{K^t}\} = \mathrm{Clust}(K^t, \mathbf{D}^t, S^t)
\end{align*}
\vspace{-.7cm}
\FOR {$j = 1,2,\ldots,K^t$}
\STATE \ \ \ \  \ \ \ $a_c = 1, \ \ \ c = \mathrm{centroid \ index \ of \ j^{th} \ cluster}$ 
\STATE \ \ \ \ \ \ \ $a_u = 0, \ \  \forall u\in Cluster_j\setminus c$
\ENDFOR
\STATE Update $S^t$ using only $K^t$ units and remove the rest
\STATE Update $\mathbf{DNN}^{t}$ by re-training the model with $K^t$ units 
\STATE $K^t=|S^t|-1$
\ENDFOR
\STATE Return pruned model with $|S^{N-1}|$ active skip-units}
\end{algorithmic}
\end{algorithm}
Next, PED only keeps $K^t$ units with the largest ED values and removes all other skip-units 
At the end of stage $t$, PED re-trains the new compressed model with weights in the active units and initialized with their values from the previous stage. This process continues until the desired size of the compressed network is met. For choosing the number of clusters ($K^t$) in each stage, while simply trying all the possible values from $1$ to $|S^t|-1$ is an acceptable approach (In ResNet or DenseNet, $|S^t|$ is at most $60$, and it decreases after each stage), one can start with a coarse pruning in the early stages and  shifts gradually to the finer pruning in the later stages.
\begin{figure}
\centering
    \begin{tabular}{cccc}
        \includegraphics[width=0.3\linewidth]{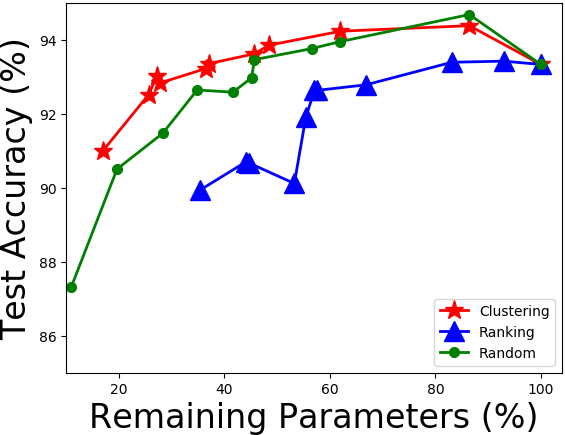} & & &
        \includegraphics[width=0.3\linewidth]{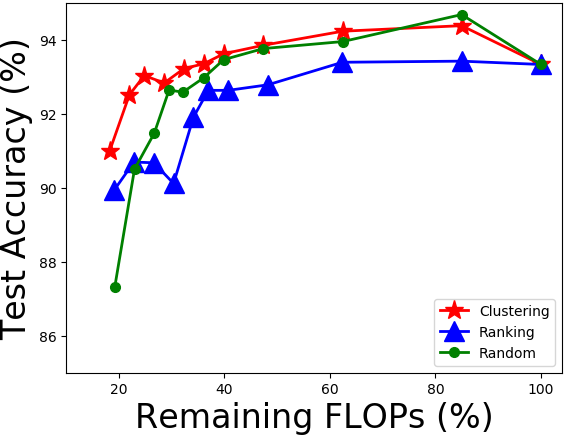}\\
        (a) & & & (b) 
    \end{tabular}
    \vspace{-.4cm}
    \caption{Comparison of clustering scheme in PED  against pruning with largest energy values, and random selection. The experiments are based on the ResNet56 and CIFAR-10.}
    \label{ranking_vs_clusterin}
    \vspace{-.2cm}
\end{figure}
An important question here is that what if we choose the units with the largest energy values, or selecting them randomly instead of clustering approach. 
We first note that the model-free measure of information is a random variable which depends on the data samples. When these values are close to each other, a test of hypothesis fails to reject if an information value is larger than the other. Thus, we cluster the values that cannot be distinguished from each other statistically. To support the preference of clustering over two other schemes, we have compared pruning using clustering with the two above schemes. Plots (a) and (b) of Figure~\ref{ranking_vs_clusterin} show the comparison for pruning ResNet56 model in classification of CIFAR-10 data set. To be fair in comparison with the clustering scheme used in PED, we have used the same number of units used in the clustering for removing units randomly and with the smallest energy. As we can see in Figure~\ref{ranking_vs_clusterin}, clustering approach results in better test accuracy both with respect to the percentage of remaining parameters, and the percentage of remaining FLOPs compared to the pruning methods with the largest Energy Dependence  and random selection specifically for the later stages.

\begin{table*}[t]
    \centering
    \begin{tabular}{c|ccccc}
    \toprule
    Models &  Acc. & Par. & Red.(\%) & FLOPs & Red.(\%)\\
    \hline
        ResNet32 (SNIP)~\cite{lee2018snip}\textsuperscript{U} & 0.9259 & 0.19 & 90.00 & - & - \\
        ResNet56 (PFEC-B)~\cite{li2016pruning}\textsuperscript{S} & 0.9306 &  0.73 & 13.70  & 90.90 & 27.60\\ 
        ResNet56 (SFP)~\cite{he2018soft}\textsuperscript{S} & 0.9359 & - & - & 59.40 & 52.60 \\ 
        ResNet56 (CNN-FCF)~\cite{li2019compressing}\textsuperscript{S} & 0.9338 & - & 43.09 & 72.40 & 42.78 \\
        ResNet56 (FPGM-mix 40\%)~\cite{he2019filter}\textsuperscript{S} & 0.9359 & - & - & 59.40 & 52.60 \\
        ResNet56 (HRank)~\cite{lin2020hrank}\textsuperscript{S} & 0.9317 & 0.49 & 42.40 &62.72 & 50.00 \\
    \hline
         ResNet56  \textbf{(ours)}&0.9336 & 0.32 & 62.96 & 45.73 & 63.86 \\
    \hline
        ResNet110 (BlockDrop)~\cite{wu2018blockdrop}\textsuperscript{S} & 0.9360 & - & - & 173.00 & 65.00\\
        ResNet110 (SkipNet)~\cite{wang2018skipnet}\textsuperscript{S} & 0.9330 & - & - & 126.00 & 50.47\\
        ResNet164-pruned~\cite{liu2017learning}\textsuperscript{S} & 0.9492 & 1.44 & 14.90 & 381.00 & 23.70\\
    \hline
         ResNet164-a  \textbf{(ours)} & 0.9519 & 0.65 & 61.78 & 71.06 & 72.13\\
         ResNet164-b \textbf{(ours)} & 0.9426 & 0.48 & 71.57 & 48.93 & 80.81\\
    \hline
    DenseNet40-pruned~\cite{liu2017learning}\textsuperscript{S} & 0.9481 & 0.66 & 35.70 & 381.00 & 28.40\\
    IGC-V2\textsuperscript{*}C416~\cite{xie2018interleaved}\textsuperscript{S} & 0.9451 & 0.65 & - & - & - \\
    CondenseNet86~\cite{huang2018condensenet}\textsuperscript{S} & 0.9496 & 0.52 & - & 65.00 & -\\
    \hline
     DenseNet100-k12-a  \textbf{(ours)} & 0.9432 & 0.29 & 61.98 & 117.73 & 59.89\\
     DenseNet100-k12-b  \textbf{(ours)} & 0.9425 & 0.27 & 64.50 & 116.51 & 60.31\\
    \bottomrule
    \end{tabular}
    \caption{Classification test accuracy, number of parameters (Par), and FLOPs on CIFAR-10 between PED and those of the state-of-the-art methods. Par and FLOPs are in million.} 
    \label{tab1:my_label}
    \vspace{-.3cm}
\end{table*}

\begin{table*}
    \centering
    \begin{tabular}{c|ccccc}
    \toprule
    Models &  Acc. & Par. & Red.(\%) & FLOPs & Red.(\%)\\
    \hline
        ResNet32 (SET)~\cite{mocanu2018scalable}\textsuperscript{U} & 0.6966 & 0.19 & 90.00 & - & - \\
        ResNet32 (GraSP)~\cite{Wang2020Picking}\textsuperscript{U} & 0.6924 & 0.19 & 90.00 & - & - \\
        Resnet110 (BlockDrop)~\cite{wu2018blockdrop}\textsuperscript{S} & 0.7370 & - & - & $\sim$284.00 & $\sim$56.00\\ 
        ResNet110 (SkipNet)~\cite{wang2018skipnet}\textsuperscript{S} & 0.7250 & - & - & - & 37.00 \\ 
        ResNet164-pruned~\cite{liu2017learning}\textsuperscript{S} & 0.7713 & 1.46 & 15.50 & 333.00 & 33.30 \\ 
    \hline
         ResNet164-a  \textbf{(ours)} & 0.7499 & 0.58 & 67.70 & 100.09 & 60.74\\
         ResNet164-b  \textbf{(ours)} & 0.7402 & 0.47 & 72.80 & 91.01 & 64.30\\
     \hline
         DenseNet40-pruned~\cite{liu2017learning}\textsuperscript{S}  & 0.7472 & 0.66 & 37.50 & 371.00 & 30.30 \\
        CondenseNet86~\cite{huang2018condensenet}\textsuperscript{S} & 0.7636 & 0.52 & - & 65.00 & -\\
     \hline
         DenseNet100-k12-a  \textbf{(ours)} & 0.7526 & 0.51 & 36.64 & 239.27 & 21.32\\
         DenseNet100-k12-b  \textbf{(ours)} & 0.7488 & 0.47 & 40.78 & 221.89 & 27.04\\
    \bottomrule
    \end{tabular}
    \caption{The top-1 test accuracy, number of parameters (Par), and FLOPs on CIFAR-100 between PED and those of the state-of-the-art methods. Par and FLOPs are in million. ``$\sim$" means approximate value. The accuracy reported in~\cite{Wang2020Picking} is given by $0. 6924\pm 0.24$.} 
    \label{tab2:my_label}
\end{table*}


\begin{table*}[!t]
\centering
\begin{tabular}{c|cccccc}
\toprule
 Model (ResNet50) &  Top1 Acc. & Top5 Acc. & Par. & Red.(\%) & FLOPs & Red.(\%)\\
\hline
  Channel Prune~\cite{he2017channel}\textsuperscript{S} & 0.7230 & 0.9080 & - & - & 5.22 & 28.00\\
  SkipNet~\cite{wang2018skipnet}\textsuperscript{S} & 0.7200 & - & - & - & -&12.00 \\ 
  SSS~\cite{huang2018data}\textsuperscript{S} & 0.7182 & 0.9079 & 15.60 & 38.82 & 2.33 & 43.32\\
  GAL-0.5~\cite{lin2019towards}\textsuperscript{S} & 0.7180 & 0.9082 & 19.31 & 24.74 & 1.84 & 55.01 \\
  HRank~\cite{lin2020hrank}\textsuperscript{S} &  0.7198 & 0.9101 & 13.77 & 46.95 & 1.55 & 62.10 \\ 
 \hline
  PED \textbf{(ours)} & 0.7280  & 0.9094 & 12.40 & 51.49 & 2.03 & 50.53 \\
\bottomrule
\end{tabular}
\caption{The top-1 and top-5 test accuracy, number of parameters (Par), and FLOPs on ImageNet data set between PED and those of the state-of-the-art methods.}
\vspace{-.4cm}
\label{ImageNet_res50main}
\end{table*}


\vspace{-.4cm}
\section{Experimental Results}
\label{experiment}
\vspace{-.2cm}
In this section, we present the performance of PED on the classification task of CIFAR-10/100~\cite{krizhevsky2010convolutional}, and ImageNet~\cite{deng2009imagenet} data sets. For the ImageNet data set, we have used the common data-augmentation scheme at training time~\cite{he2016deep}, and perform a re-scaling to $256\times256$ followed by a $224\times224$ center crop at test time.
We have compared the performance of PED with the state-of-the-art methods for CIFAR-10, CIFAR-100, and ImageNet datsets respectively in Table~\ref{tab1:my_label}, Table~\ref{tab2:my_label}, and Table~\ref{ImageNet_res50main}. The third/fourth and fifth/sixth columns labeled 'Red' in all tables represent the percentage of reduction in the number of parameters and FLOPs, respectively. Also, the numbers given in columns labeled as "Par." (i.e., Parameters) and "FLOPs" are in million (Only FLOPs in Table~\ref{ImageNet_res50main} is in Billion) and rounded by two-decimal digits. In all tables, ``–" means no reported value, and the superscript ``S" and ``U" denote the structured and unstructured methods, respectively. 

The ResNet56 architecture consists of $56$ layers with $27$ residual units and the total number of $0.85$ (M) training parameters and $126.54$ (M) FLOPs. In order to run PED on ResNet56, we first train it on CIFAR-10 to achieve $0.9334$ test accuracy. The results of applying PED on ResNet56 with $N=5$ stages are shown in Table~\ref{tab1:my_label}. As we can see, without almost any dropping in the test accuracy, we can reduce the number of trainable parameters and FLOPs by $62.96\%$ and $63.86\%$, respectively. We next consider the ResNet164 which consists of $18$ units with total number of $1.70$ (M) trainable parameters and $254.94$ (M) FLOPs. Similar to the ResNet56, we first train this model on the CIFAR-10 and achieve $0.9569$ test accuracy. In Table~\ref{tab1:my_label}, we have listed two ResNet164 pruned models. The first one, ResNet164-a is a compressed version of the ResNet164 by running PED for $N=6$ stages, while ResNet164-b corresponds to $N=9$ stages. As we can see, ResNet164-a achieves $0.9519$ accuracy with $0.65$ (M) parameters and $71.06$ (M) FLOPs. For the CIFAR-100, a full ResNet164 achieves $77.93\%$ top-1 accuracy. Table~\ref{tab2:my_label} also presents ResNet164-a and ResNet164-b corresponding to $N=9$ and $N=11$ stages, respectively. 

Next we focus on the DenseNet100-k12 with $100$ layers, $k=12$  growth-rate (i.e., the number of output channels in each unit), and $48$ skip-units. We apply PED to DenseNet100-k12 by first training DenseNet100-k12 on both the CIFAR-10 with $0.9531$ test accuracy, $0.77$ (M) number of parameters, and $293.55$ (M) FLOPs (Tabel~\ref{tab1:my_label}) and the CIFAR-100 with $0.7793$ as the top-1 test accuracy, $0.80$ (M) number of parameters, and $304.10$ (M) FLOPs (Table~\ref{tab2:my_label}). DenseNet100-k12-a and DenseNet100-k12-b correspond to running PED with $N=10$ and $N=11$ stages, respectively. Also we have two models of DenseNet100-k12-a and DenseNet100-k12-b in Table~\ref{tab2:my_label} corresponding to running PED with $N=5$ and $N=6$ stages, respectively. These experiments suggest that our proposed algorithm is competitive to the state-of-the-art methods in one or more criteria of accuracy/parameters/FLOPs. Finally, Table~\ref{ImageNet_res50main} illustrates experiments on the ImagaNet data set using the ResNet50. The model consists of 4 blocks with 3, 4, 6, and 3 skip-units in each block. ResNet50 has $25.56$ (M) trainable parameters and $4.11$ (B) FLOPs. Compared to methods of SkipNet, SSS, GAL-0.5, and HRank, our method has better performance.

\vspace{-.2cm}
\Section{References}
\vspace{-.3cm}
\bibliographystyle{IEEEbib}
\bibliography{mrsbib}

\end{document}